\documentclass[letterpaper]{article} 
\usepackage{aaai23}  
\usepackage{times}  
\usepackage{helvet}  
\usepackage{courier}  
\usepackage[hyphens]{url}  
\usepackage{graphicx} 
\usepackage{natbib}  
\usepackage{caption} 
\usepackage[utf8]{inputenc}
\usepackage[T1]{fontenc}
\usepackage{algorithm}
\usepackage{algcompatible}%

\usepackage{url}            
\usepackage{booktabs}       
\usepackage{amsfonts}       
\usepackage{nicefrac}       
\usepackage{microtype}      
\usepackage{xcolor}         

\usepackage{multicol}
\usepackage{multirow}
\usepackage{amsfonts}
\usepackage{breqn}
\usepackage{stmaryrd}
\usepackage{algcompatible}

\usepackage{amssymb}
\usepackage{amsthm}
\usepackage{amsmath}
\usepackage{dsfont}
\usepackage{paralist}
\usepackage{wrapfig}
\newtheorem{definition}{Definition}
\newtheorem{theorem}{Theorem}

\newtheorem{proposition}{Proposition}

\newcommand{\Init}{\mathcal{X}_0}
\newcommand{\Unsafe}{\mathcal{X}_u}

\newcommand{\eps}{\epsilon}

\newcommand{\Ind}{\mathds{1}}
\newcommand{\loss}{\mathcal{L}}
\newcommand{\support}{\mathsf{support}}
\newcommand{\Target}{\mathcal{X}_t}

\newcommand{\ReachSafe}{\textrm{ReachAvoid}}

\newcommand{\Discretization}{\tilde{\mathcal{X}}}

\usepackage{newfloat}
\usepackage{listings}
\DeclareCaptionStyle{ruled}{labelfont=normalfont,labelsep=colon,strut=off} 
\floatstyle{ruled}
\newfloat{listing}{tb}{lst}{}
\floatname{listing}{Listing}
\title{Learning Control Policies for Stochastic Systems with Reach-avoid Guarantees}
\author{
\DJ{}or\dj{}e \v{Z}ikeli\'c$^{\star 1}$, 
Mathias Lechner$^{\star 2}$,
Thomas A. Henzinger$^{1}$,
Krishnendu Chatterjee$^{1}$
}
\affiliations{
    \textsuperscript{\rm 1}Institute of Science and Technology Austria (ISTA)\\
    \textsuperscript{\rm 2}Massachusetts Institute of Technology (MIT)\\
    \textsuperscript{\rm $\star$}Equal contribution\\

    Correspondence to djordje.zikelic@ist.ac.at, mlechner@mit.edu
}

\usepackage{bibentry}

\begin{document}

\maketitle

\begin{abstract}

We study the problem of learning controllers for discrete-time non-linear stochastic dynamical systems with formal reach-avoid guarantees. This work presents the first method for providing formal reach-avoid guarantees, which combine and generalize stability and safety guarantees, with a tolerable probability threshold $p\in[0,1]$ over the infinite time horizon in general Lipschitz continuous systems.
Our method leverages advances in machine learning literature and it represents formal certificates as neural networks.
In particular, we learn a certificate in the form of a reach-avoid supermartingale (RASM), a novel notion that we introduce in this work. Our RASMs provide reachability and avoidance guarantees by imposing constraints on what can be viewed as a stochastic extension of level sets of Lyapunov functions for deterministic systems. Our approach solves several important problems -- it can be used to learn a control policy from scratch, to verify a reach-avoid specification for a fixed control policy, or to fine-tune a pre-trained policy if it does not satisfy the reach-avoid specification. We validate our approach on $3$ stochastic non-linear reinforcement learning tasks.

\end{abstract}

\section{Introduction}\label{sec:intro}

Reinforcement learning (RL) has achieved impressive results in solving non-linear control problems, resulting in an interest to deploy RL algorithms in safety-critical applications. However, most RL algorithms focus solely on optimizing expected performance and do not take safety constraints into account~\citep{sutton2018reinforcement}. This raises concerns about their applicability to safety-critical domains in which unsafe behavior can lead to catastrophic consequences~\citep{AmodeiOSCSM16,GarciaF15}. Complicating matters, models are usually imperfect approximations of real systems that are obtained from observed data, thus models often need to account for uncertainty which is modelled via stochastic disturbances. Formal safety verification of policies learned via RL algorithms and design of learning algorithms that take safety constraints into account have thus become very active research topics.

{\em Reach-avoid constraints} are one of the most common and practically relevant constraints appearing in safety-critical applications that generalize both reachability and safety constraints~\citep{SummersL10}. Given a target region and an unsafe region, the reach-avoid constraint requires that a system controlled by a policy converges to the target region while avoiding the unsafe region. For instance, a lane-keeping constraint requires a self-driving car to reach its destination without leaving the allowed car lanes~\citep{VahidiE03}. In the case of stochastic control problems, reach-avoid constraints are also specified by a minimal probability with which the system controlled by a policy needs to satisfy the reach-avoid constraint.

In this work, we consider discrete-time stochastic control problems under reach-avoid constraints. Following the recent trend that aims to leverage advances in deep RL to safe control, we propose a learning method that learns a control policy together with a formal reach-avoid certificate in the form of a {\em reach-avoid supermartingale (RASM)}, a novel notion that we introduce in this work. 
Informally, an RASM is a function assigning nonnegative real values to each state that is required to strictly decrease in expected value until the target region is reached, but needs to strictly increase for the system to reach the unsafe region. By carefully choosing the ratio of the initial level set of the RASM and the least level set that the RASM needs to attain for the system to reach the unsafe region (here we use the standard level set terminology of Lyapunov functions~\citep{haddad2011nonlinear}), we obtain a formal reach-avoid certificate. The name of RASMs is chosen to emphasize the connection to supermartingale processes in probability theory~\citep{Williams91}. Our RASMs significantly generalize and unify the stochastic control barrier functions which are a standard certificate for safe control of stochastic systems~\citep{PrajnaJP07} and ranking supermartingales that were introduced to certify probability $1$ reachability and stability in~\citep{lechner2021stability}.

{\bf Contributions.} This work presents the first control method that provides formal {\em reach-avoid guarantees} for control of stochastic systems with a {\em specified probability threshold} over the {\em infinite time horizon} in {\em Lipschitz continuous} systems. In contrast, the existing approaches to control under reach-avoid constraints are only applicable to finite horizon settings, polynomial stochastic systems or to deterministic systems (see the following section for an overview of related work). 
Moreover, our method simultaneously {\em learns} the control policy and the RASM certificate in the form of neural networks and is applicable to general non-linear systems. This contrasts the existing methods from the literature that are based on stochastic control barrier functions, which utilize convex optimization tools to compute control policies and are restricted to polynomial system dynamics and policies~\citep{PrajnaJP07,SteinhardtT12,SantoyoDC21,XLZF21}. 
Our algorithm draws insight from established methods for learning Lyapunov functions for stability in deterministic control problems~\citep{RichardsB018,ChangRG19,AbateAGP21}, which were demonstrated to be more efficient than the existing convex optimization methods and were adapted in~\citep{lechner2021stability} for probability~$1$ reachability and stability verification. 
Finally, our method learns a suitable policy on demand, or alternatively, {\em verifies} reach-avoid properties of a fixed Lipschitz continuous control policy. We experimentally validate our method on $3$ stochastic RL tasks and show that it efficiently learns control policies with probabilistic reach-avoid guarantees in practice.

\section{Related Work}\label{sec:relatedwork}

\noindent{\bf Deterministic control problems} There is extensive literature on safe control, with most works certifying stability via Lyapunov functions~\citep{haddad2011nonlinear} or safety via control barrier functions~\citep{AmesCENST19}. Most early works rely either on hand-designed certificates, or automate their computation through convex optimization methods such as sum-of-squares (SOS) programming~\citep{henrion2005positive,parrilo2000structured,jarvis2003some}. Automation via SOS programming is restricted to problems with polynomial system dynamics and does not scale well with dimension. A promising approach to overcome these limitations is to learn a control policy together with a safety certificate in the form of neural networks, for instance see~\citep{RichardsB018,SunJF20,Jin20,ChangG21,QinZCCF21}. In particular, \citep{ChangRG19,AbateAGP21} learn a control policy and a certificate as neural networks by using a learner-verifier framework which repeatedly learns a candidate policy and a certificate 
and then tries to either verify or refine them. Our method extends some of these ideas to stochastic systems.

\smallskip\noindent{\bf Stochastic control problems} Safe control of stochastic systems has received comparatively less attention. Most existing approaches are abstraction based -- they consider finite-time horizon systems and approximate them via a finite-state Markov decision process (MDP). The constrained control problem is then solved for the MDP. 
Due to accumulation of the approximation error in each time step, the size of the MDP state space needs to grow with the length of the considered time horizon, making these methods applicable to systems that evolve over fixed finite time horizons. Notable examples include~\citep{SoudjaniGA15,LavaeiKSZ20,cauchi2019stochy,VinodGO19,Vaidya15,CrespoS03}.
Another line of work considers polynomial systems and utilizes stochastic control barrier functions and convex optimization tools to compute polynomial control policies~\citep{PrajnaJP07,SteinhardtT12,SantoyoDC21,XLZF21}.

\smallskip\noindent{\bf Constrained MDPs} Safe RL has also been studied in the context of constrained MDPs (CMDPs)~\citep{altman1999constrained,Geibel06}. An agent in a CMDP must satisfy hard constraints on expected cost for one or more auxiliary notions of cost aggregated over an episode. Several works study RL algorithms for CMDPs~\citep{uchibe2007constrained}, notably the Constrained Policy Optimization (CPO)~\citep{achiam2017constrained} or the method~\cite{ChowNDG18} which proposed a Lyapunov method for solving CMDPs.
While these algorithms perform well, their constraints are satisfied in expectation which makes them less suitable for safety-critical systems.

\smallskip\noindent{\bf Safe RL via shielding} Some approaches ensure safety by computing two control policies -- the main policy that optimizes the expected reward, and the backup policy that the system falls back to whenever a safety constraint may be violated~\citep{MichalskaM93,PerkinsB02,AlshiekhBEKNT18,Elsayed-AlyBAET21,GiacobbeHKW21}. The backup policy can thus be of simpler form. 
Shielding for stochastic linear systems with additive disturbances has been considered in~\citep{WabersichZ18}. \citep{LiB20,BastaniL21} are applicable to stochastic non-linear systems, however their safety guarantees are {\em statistical} -- their algorithms are randomized with parameters $\delta,\eps\in(0,1)$ and they with probability $1-\delta$ compute an action that is safe in the current state with probability at least $1-\eps$.
The statistical error is accumulated at each state, hence these approaches are not suitable for infinite or long time horizons. In contrast, our approach targets {\em formal} guarantees for {\em infinite} time horizon problems.


\smallskip\noindent{\bf Safe exploration} Model-free RL algorithms need to explore the state space in order to learn high performing actions. 
Safe exploration RL restricts exploration in a way which ensures that given safety constraints are satisfied.
The most common approach to ensuring safe exploration is learning the system dynamics' uncertainty bounds and limiting the exploratory actions within a high probability safety region, with the existing methods based on Gaussian Processes~\citep{Koller2018LearningBasedMP, Turchetta2019SafeEF, Berkenkamp2019SafeEI}, linearized models~\cite{Dalal2018SafeEI}, deep robust regression~\citep{Liu2020RobustRF}, safe padding~\cite{HasanbeigAK20} and Bayesian neural networks~\citep{lechner2021infinite}.
Recent work has also considered learning stable stochastic dynamics from data~\citep{UmlauftH17,LawrenceLFBG20}.

\smallskip\noindent{\bf Probabilistic program analysis} Supermartingales have also been used for the analysis of probabilistic programs (PPs).
In particular, RSMs were originally used to prove almost-sure termination in PPs~\citep{ChakarovS13} and~\citep{AbateGR20} learns RSMs in PPs. Supermartingales were also used for probabilistic termination and safety analysis in PPs~\citep{ChatterjeeNZ17,ChatterjeeGMZ22}.

\section{Preliminaries}\label{sec:prelims}

We consider discrete-time stochastic dynamical systems defined by the equation
\[ \mathbf{x}_{t+1} = f(\mathbf{x}_t, \mathbf{u}_t, \omega_t),\,\mathbf{x}_0\in\Init. \]
The function $f:\mathcal{X}\times\mathcal{U}\times\mathcal{N}\rightarrow \mathcal{X}$ defines system dynamics, where $\mathcal{X}\subseteq\mathbb{R}^m$ is the system state space, $\mathcal{U}\subseteq\mathbb{R}^n$ is the control action space and $\mathcal{N}\subseteq\mathbb{R}^p$ is the stochastic disturbance space. We use $t\in\mathbb{N}_0$ to denote the time index, $\mathbf{x}_t\in\mathcal{X}$ the state of the system, $\mathbf{u}_t\in\mathcal{U}$ the action and $\omega_t\in\mathcal{N}$ the stochastic disturbance vector at time $t$. The set $\Init\subseteq \mathcal{X}$ is the set of initial states. The action $\mathbf{u}_t$ is chosen according to a control policy $\pi:\mathcal{X}\rightarrow\mathcal{U}$, i.e.~$\mathbf{u}_t = \pi(\mathbf{x}_t)$. The stochastic disturbance vector $\omega_t$ is sampled according to a specified probability distribution $d$ over $\mathbb{R}^p$. The dynamics function $f$, control policy $\pi$ and probability distribution $d$ together define a stochastic feedback loop system.

A sequence $(\mathbf{x}_t,\mathbf{u}_t,\omega_t)_{t\in\mathbb{N}_0}$  of state-action-disturbance triples is a trajectory of the system, if for each $t\in\mathbb{N}_0$ we have $\mathbf{u}_t=\pi(\mathbf{x}_t)$, $\omega_t\in\support(d)$ and $\mathbf{x}_{t+1}=f(\mathbf{x}_t,\mathbf{u}_t,\omega_t)$. For each initial state $\mathbf{x}_0\in\mathcal{X}$, the system induces a Markov process which gives rise to the probability space over the set of all trajectories that start in $\mathbf{x}_0$~\citep{Puterman94}. We denote the probability measure and the expectation in this probability space by $\mathbb{P}_{\mathbf{x}_0}$ and $\mathbb{E}_{\mathbf{x}_0}$.

\smallskip\noindent{\bf Assumptions} We assume that $\mathcal{X}\subseteq\mathbb{R}^m$, $\Init\subseteq\mathbb{R}^m$, $\mathcal{U}\subseteq\mathbb{R}^n$ and $\mathcal{N}\subseteq\mathbb{R}^p$ are all Borel-measurable, which is a technical assumption necessary for the system semantics to be mathematically well-defined. 
We also assume that $\mathcal{X}\subseteq\mathbb{R}^m$ is compact 
and that the dynamics function $f$ is Lipschitz continuous, which are common assumptions in control theory.

\smallskip\noindent{\bf Probabilistic reach-avoid problem} 
Let $\mathcal{X}_t\subseteq\mathcal{X}$ and $\mathcal{X}_u\subseteq\mathcal{X}$ be disjoint Borel-measurable subsets of $\mathbb{R}^m$, which we refer to as the {\em target set} and the {\em unsafe set}, respectively. Let $p\in [0,1]$ be a probability threshold. Our goal is to learn a control policy which guarantees that, with probability at least $p$, the system reaches the target set $\Target$ without reaching the unsafe set $\Unsafe$. Formally, we want to learn a control policy $\pi$ such that, for any initial state $\mathbf{x}_0\in\Init$, we have
\[ \mathbb{P}_{\mathbf{x}_0}\Big[ \ReachSafe(\Target,\Unsafe)\Big] \geq p \]
with $\ReachSafe(\Target,\Unsafe) = \{ (\mathbf{x}_t,\mathbf{u}_t,\omega_t)_{t\in\mathbb{N}_0} \mid \exists t\in\mathbb{N}_0.\, \mathbf{x}_t\in \Target \land (\forall t'\leq t.\, \mathbf{x}_{t'}\not\in \Unsafe)\}$ the set of trajectories that reach $\Target$ without reaching $\Unsafe$.

We restrict to the cases when either $p<1$, or $p=1$ and $\Unsafe=\emptyset$. Our approach is not applicable to the case $p=1$ and $\Unsafe\neq\emptyset$ due to technical issues that arise in defining our formal certificate, which we discuss in the following section.
We remark that probabilistic reachability is a special instance of our problem 
obtained by setting $\Unsafe=\emptyset$. On the other hand, we cannot directly obtain the probabilistic safety problem by assuming any specific form of the target set $\Target$, however we will show in the following section that our method implies probabilistic safety with respect to $\Unsafe$ if we provide it with $\Target=\emptyset$.

\section{Theoretical Results}\label{sec:theoretical}

We now present our framework for formally certifying a reach-avoid constraint with a given probability threshold. Our framework is based on the novel notion of {\em reach-avoid supermartingales (RASMs)} that we introduce in this work. Note that, in this section only, we assume that the policy is fixed. In the next section, we will present our algorithm for learning policies that provide formal reach-avoid guarantees in which RASMs will be an integral ingredient.
In what follows, we consider a discrete-time stochastic dynamical system defined as in the previous section. For now, we assume that the probability threshold is strictly smaller than $1$, i.e.~$p<1$. We will later show that our approach straightforwardly extends to the case $p=1$ and $\Unsafe=\emptyset$.

\smallskip\noindent{\bf Reach-avoid supermartingales} We define a {\em reach-avoid supermartingale (RASM)} to be a continuous function $V:\mathcal{X}\rightarrow \mathbb{R}$ that assigns real values to system states. The name is chosen to emphasize the connection to supermartingale processes from probability theory~\citep{Williams91}, which we will explore later in order to prove the effectiveness of RASMs for verifying reach-avoid properties. The value of $V$ is required to be nonnegative over the state space $\mathcal{X}$ (Nonnegativity condition), to be bounded from above by $1$ over the set of initial states $\Init$ (Initial condition) and to be bounded from below by $\frac{1}{1-p}$ over the set of unsafe states $\Unsafe$ (Safety condition). Hence, in order for a system trajectory to reach an unsafe state and violate the safety specification, the value of the RASM $V$ needs to increase at least $\frac{1}{1-p}$ times along the trajectory. Finally, we require the existence of $\eps>0$ such that the value of $V$ decreases in expected value by at least $\eps$ after every one-step evolution of the system from every system state $\mathbf{x}\in\mathcal{X}\backslash\Target$ for which $V(\mathbf{x}) \leq \frac{1}{1-p}$ (Expected decrease condition). Intuitively, this last condition imposes that the system has a tendency to strictly {\em decrease} the value of $V$ until either the target set $\Target$ is reached or a state with $V(\mathbf{x})\geq \frac{1}{1-p}$ is reached. However, as the value of $V$ needs to {\em increase} at least $\frac{1}{1-p}$ times in order for the system to reach an unsafe state, these four conditions will allow us to use RASMs to certify that the reach-avoid constraint is satisfied with probability at least $p$.

\begin{definition}[Reach-avoid supermartingales]\label{def:rssm}
Let $\mathcal{X}_t\subseteq\mathcal{X}$ and $\mathcal{X}_u\subseteq\mathcal{X}$ be the target set and the unsafe set, and let $p\in[0,1)$ be the probability threshold. A continuous function $V:\mathcal{X}\rightarrow \mathbb{R}$ is said to be a {\em reach-avoid supermartingale (RASM)} with respect to $\Target$, $\Unsafe$ and $p$ if it satisfies:
\begin{compactenum}
    \item {\em Nonnegativity condition.} $V(\mathbf{x}) \geq 0$ for each $\mathbf{x}\in\mathcal{X}$.
    \item {\em Initial condition.} $V(\mathbf{x}) \leq 1$ for each $\mathbf{x}\in\Init$.
    \item {\em Safety condition.} $V(\mathbf{x}) \geq \frac{1}{1-p}$ for each $\mathbf{x}\in\Unsafe$.
    \item {\em Expected decrease condition.} There exists $\eps>0$ such that, for each $\mathbf{x}\in\mathcal{X}\backslash\Target$ at which $V(\mathbf{x}) \leq \frac{1}{1-p}$, we have $V(\mathbf{x}) \geq \mathbb{E}_{\omega\sim d}[V(f(\mathbf{x},\pi(\mathbf{x}),\omega))] + \eps$.
\end{compactenum}
\end{definition}

\noindent{\bf Comparison to Lyapunov functions} The defining properties of RASMs hint a connection to Lyapunov functions for deterministic control systems. However, the key difference between Lyapunov functions and our RASMs is that Lyapunov functions deterministically decrease in value whereas RASMs decrease in expectation. Deterministic decrease ensures that each level set of a Lyapunov function, i.e.~a set of states at which the value of Lyapunov functions is at most $l$ for some $l\geq 0$, is an invariant of the system. However, it is in general not possible to impose such a condition on stochastic systems. In contrast, our RASMs only require expected decrease in the level, and the Initial and the Unsafe conditions can be viewed as conditions on the {\em maximal initial level set} and the {\em minimal unsafe level set}. The choice of a ratio of these two level values allows us to use existing results from martingale theory in order to obtain probabilistic avoidance guarantees, while the Expected decrease condition by $\eps>0$ furthermore provides us with probabilistic reachability guarantees.

\smallskip\noindent{\bf Certifying reach-avoid constraints via RASMs} We now show that the existence of an $\eps$-RASM for some $\eps>0$ implies that the reach-avoid constraint is satisfied with probability at least $p$. 

\begin{theorem}\label{thm:rssmreachsafe}
Let $\mathcal{X}_t\subseteq\mathcal{X}$ and $\mathcal{X}_u\subseteq\mathcal{X}$ be the target set and the unsafe set, respectively, and let $p\in[0,1)$ be the probability threshold. Suppose that there exists an RASM $V$  with respect to $\Target$, $\Unsafe$ and $p$. Then, for every $\mathbf{x}_0\in\Init$, $\mathbb{P}_{\mathbf{x}_0}[ \ReachSafe(\Target,\Unsafe)] \geq p$.
\end{theorem}

The complete proof of Theorem~\ref{thm:rssmreachsafe} is provided in the Appendix. However, in what follows we sketch the key ideas behind our proof, in order to illustrate the applicability of martingale theory to reasoning about stochastic systems which we believe to have significant potential for applications beyond the scope of this work. To prove the theorem, we first show that an $\eps$-RASM $V$ induces a {\em supermartingale}~\citep{Williams91} in the probability space over the set of all trajectories that start in an initial state $\mathbf{x}_0\in\Init$. Intuitively, a supermartingale in a probability space $(\Omega,\mathcal{F},\mathbb{P})$ is a stochastic process $(X_t)_{t=0}^\infty$ such that, for each $t\in\mathbb{N}_0$, the expected value of $X_{t+1}$ conditioned on the value of $X_t$ is less than or equal to $X_t$. We formalize this definition together with the notion of conditional expectation and provide an overview of definitions and results form martingale theory that we use in our proof in the Appendix.

Now, let $(\Omega_{\mathbf{x}_0},\mathcal{F}_{\mathbf{x}_0},\mathbb{P}_{\mathbf{x}_0})$ be the probability space of trajectories that start in $\mathbf{x}_0$. Then, for each time step $t\in\mathbb{N}_0$, we define a random variable
\begin{equation*}
X_t(\rho) = \begin{cases}
    V(\mathbf{x}_t), &\text{if } \mathbf{x}_i\not\in\Target \text{ and } V(\mathbf{x}_i) < \frac{1}{1-p} \\
    &\text{ for each } 0\leq i \leq t\\
    0, &\text{if } \mathbf{x}_i\in\Target \text{ for some }0\leq i\leq t \\
    &\text{ and } V(\mathbf{x}_j) < \frac{1}{1-p} \text{ for each }0\leq j\leq i\\
    \frac{1}{1-p}, &\text{otherwise}
\end{cases}
\end{equation*}
for each trajectory $\rho=(\mathbf{x}_t,\mathbf{u}_t,\omega_t)_{t\in\mathbb{N}_0}\in\Omega_{\mathbf{x}_0}$. In other words, the value of $X_t$ is equal to the value of $V$ at $\mathbf{x}_t$, unless either the target set $\Target$ has been reached first in which case we set all future values of $\mathcal{X}_t$ to $0$, or a state in which $V$ exceeds $\frac{1}{1-p}$ has been reached first in which case we set all future values of $\mathcal{X}_t$ to $\frac{1}{1-p}$. Then, since $V$ satisfies the Nonnegativity and the Expected decrease condition of RASMs, we may show that $(X_t)_{t=0}^\infty$ is a supermartingale. in the probability space $(\Omega_{\mathbf{x}_0},\mathcal{F}_{\mathbf{x}_0},\mathbb{P}_{\mathbf{x}_0})$.

Next, we show that the nonnegative supermartingale $(X_t)_{t=0}^\infty$ with probability $1$ converges to and reaches $0$ or a value that is greater than or equal to $\frac{1}{1-p}$. To do this, we first employ the Supermartingale Convergence Theorem (see the Appendix) which states that every nonnegative supermartingale converges to some value with probability $1$. We then use the fact that, in the Expected decrease condition of RASMs, the decrease in expected value is strict and by at least $\eps>0$, in order to conclude that this value is reached and has to be either $0$ or greater than or equal to $\frac{1}{1-p}$.

Finally, we use another classical result from martingale theory (see the Appendix) which states that, given a nonnegative supermartingale $(X_t)_{t=0}^\infty$ and $\lambda>0$,
\[ \mathbb{P}\Big[ \sup_{i\geq 0}X_i \geq \lambda \Big] \leq \frac{\mathbb{E}[X_0]}{\lambda}. \]
Plugging $\lambda=\frac{1}{1-p}$ into the above inequality, it follows that $\mathbb{P}_{\mathbf{x}_0}[ \sup_{i\geq 0}X_i \geq \frac{1}{1-p} ] \leq (1-p)\cdot \mathbb{E}_{\mathbf{x}_0}[X_0] \leq 1-p$. The second inequality follows since $X_0(\rho) = V(\mathbf{x}_0)\leq 1$ for every $\rho\in\Omega_{\mathbf{x}_0}$ by the Initial condition of RASMs. Hence, as $(X_t)_{t=0}^\infty$ with probability $1$ either reaches $0$ or a value that is greater than or equal to $\frac{1}{1-p}$, we conclude that $(X_t)_{t=0}^\infty$ reaches $0$ without reaching a value that is greater than or equal to $\frac{1}{1-p}$ with probability at least $p$. By the definition of each $X_t$ and by the Safety condition of RASMs, this implies that with probability at least $p$ the system will reach the target set $\Target$ without reaching the unsafe set $\Unsafe$, i.e.~that $\mathbb{P}_{\mathbf{x}_0}[ \ReachSafe(\Target,\Unsafe) ] \geq p$.

\smallskip\noindent{\bf Probabilistic safety} In order to solve the probabilistic safety problem and verify that a control policy guarantees that the unsafe set $\Unsafe$ is not reached with probability at least $p$, we may modify the Expected decrease condition of RASMs by setting $\Target=\emptyset$. Thus, RASMs are also effective for the probabilistic safety problem. This claim follows immediately from our proof of Theorem~\ref{thm:rssmreachsafe}. In this case and if we set $\eps=0$, then our RASMs coincide with stochastic barrier functions of~\citep{PrajnaJP07}. However, if $\Target$ is not empty, then we must have $\eps>0$ in order to enforce convergence and reachability of $\Target$.

\smallskip\noindent{\bf Extension to $p=1$ and $\Unsafe=\emptyset$ and comparison to RSMs} So far, we have only considered $p\in[0,1)$. The difficulty in the case $p=1$ arises since the value $\frac{1}{1-p}$ in the Safety and the Expected decrease conditions in Definition~\ref{def:rssm} would not be well-defined. However, if $\Unsafe=\emptyset$, then the Safety condition need not be imposed at any state. Moreover, it follows directly from our proof that imposing the expected decrease condition at all states in $\mathcal{X}\backslash\Target$ makes RASMs sound for certifying probability~$1$ reachability. In fact, in this special case our RASMs reduce to the RSMs of~\citep{lechner2021stability}. The key novelty of our RASMs over RSMs is that we also employ {\em level set reasoning} in order to obtain probabilistic reach-avoid guarantees, thus presenting a true {\em stochastic extension of Lyapunov functions} that allow reasoning both about reach-avoid specifications as well as quantitative reasoning about the probability with which they are satisfied. In contrast, RSMs do not reason about level sets and can only certify probability~$1$ reachability.

\section{Learning Reach-avoid Policies}\label{sec:algo}

We now present our algorithm for learning policies with reach-avoid guarantees, which learns a policy together with an RASM certificate. The algorithm consists of two modules called {\em learner} and {\em verifier}, which are composed into a loop. In each loop iteration, the learner learns a policy together with an RASM candidate as two neural networks $\pi_{\theta}$ and $V_{\nu}$, with $\theta$ and $\nu$ being vectors of neural network parameters. The verifier then formally verifies whether the learned RASM candidate is indeed an RASM for the system and the learned policy. If the answer is positive, then the algorithm concludes that the learned policy provides formal reach-avoid guarantees. Otherwise, the verifier computes a counterexample which shows that the learned RASM candidate is not an RASM. The counterexample is passed to the learner and used to modify the loss function towards learning a new policy and an RASM candidate. The loop is repeated until either a candidate is successfully verified or the algorithm reaches a specified timeout. The algorithm is presented in Algorithm~\ref{alg:algorithm}. 
We note that our algorithm can also {\em verify} whether a given Lipschitz continuous policy provides reach-avoid guarantees, by fixing the policy only learning the RASM neural network.

\begin{algorithm}[t]
\caption{Algorithm for learning reach-avoid policies}
\label{alg:algorithm}
\begin{algorithmic}[1]
\STATE \textbf{Input} $f$, $d$, $\mathcal{X}$, $\Init,\Target,\Unsafe$, $L_f$, $p\in[0,1]$
\STATE \textbf{Parameters} mesh $\tau>0$, number of samples $N\in\mathbb{N}$, regularization constant $\lambda > 0$

\STATE $\pi_\theta \leftarrow$ trained by PPO
\STATE $\tilde{\mathcal{X}} \leftarrow $ discretization of $\mathcal{X}$ with mesh $\tau$
\STATE $C_{\text{init}},\, C_{\text{unsafe}},\, C_{\text{decrease}} \leftarrow \tilde{\mathcal{X}}\cap\Init,\, \tilde{\mathcal{X}}\cap\Unsafe, \, \tilde{\mathcal{X}}\cap(\mathcal{X}\backslash\Target)$
\STATE $V_\nu \leftarrow$ trained by minimizing the loss function
\WHILE{timeout not reached}

\STATE $L_\pi, L_V \leftarrow$ Lipschitz constants of $\pi_\theta$, $V_\nu$
\STATE $K\leftarrow L_V \cdot (L_f \cdot (L_\pi + 1) + 1)$
\STATE $\tilde{\mathcal{X}_e} \leftarrow$ vertices of discr. $\tilde{\mathcal{X}}$ whose adjacent cells intersect $\mathcal{X}\backslash\Target$ and contain $\mathbf{x}$ s.t.~$V_\nu(\mathbf{x})<\frac{1}{1-p}$
\STATE $\text{Cells}_{\Init},\, \text{Cells}_{\Unsafe} \leftarrow$ discr. cells that intersect $\Init,\, \Unsafe$
\IF{$\exists \tilde{\mathbf{x}}\in\tilde{\mathcal{X}_e} \cap (\mathcal{X}\backslash\Target)$ s.t.~$\mathbb{E}_{\omega\sim d}[ V_\nu ( f(\tilde{\mathbf{x}}, \pi(\tilde{\mathbf{x}}), \omega) ) ] \geq V_\nu(\tilde{\mathbf{x}}) - \tau \cdot K$ and $V_\nu(\tilde{\mathbf{x}})<\frac{1}{1-p}$}
\STATE $C_{\text{decrease}} \leftarrow C_{\text{decrease}} \cup \{\mathbf{x}\}$
\ELSIF{$\exists \text{cell} \in \text{Cells}_{\Init}$ s.t.~$\sup_{\mathbf{x}\in\text{cell}}V_\nu(\mathbf{x}) > 1$}
\STATE $C_{\text{init}} \leftarrow C_{\text{init}} \cup (\{\text{vertices of cell}\} \cap\Init)$
\ELSIF{$\exists \text{cell} \in \text{Cells}_{\Unsafe}$ s.t.~$\inf_{\mathbf{x}\in\text{cell}}V_\nu(\mathbf{x}) < \frac{1}{1-p}$}
\STATE $C_{\text{unsafe}} \leftarrow C_{\text{unsafe}} \cup (\{\text{vertices of cell}\} \cap\Unsafe)$
\ELSE
\STATE \textbf{Return} Reach-avoid guarantee with probability $p$
\ENDIF
\STATE $V_\nu, \pi_\theta,  \leftarrow$ trained by minimizing the loss function
\STATE $\tilde{\mathcal{X}} \leftarrow$ refined discretization
\ENDWHILE
\STATE \textbf{Return} Unknown
\end{algorithmic}
\end{algorithm}

\smallskip\noindent{\bf Policy Initialization}
Learning two networks concurrently with multiple objectives can be unstable due to dependencies between the two networks and differences in the scale of the objective loss terms.
To mitigate these instabilities, we propose pre-training of the policy network so that our algorithm starts from a proper initialization.
In particular, from the given dynamical system and the safety specification, we induce a Markov decision process (MDP) intending to reach the target set while avoiding the unsafe set.
The reward term $r_t$ is given by $r_t := \Ind[\Target](\mathbf{x}_t) - \Ind[\Unsafe](\mathbf{x}_t)$ and we use proximal policy optimization (PPO)~\citep{schulman2017proximal} to train the policy.

\smallskip\noindent{\bf State Space Discretization} When it comes to verifying learned candidates, the key difficulty lies in checking the Expected decrease condition. This is because, in general, it is not possible to compute a closed form expression for the expected value of an RASM over successor system states, as both the policy and the RASM are neural networks. In order to overcome this difficulty, our algorithm discretizes the state space of the system. Given a {\em mesh} parameter $\tau>0$, a {\em discretization $\Discretization$} of $\mathcal{X}$ with mesh $\tau$ is a set of states such that, for every $\mathbf{x}\in \mathcal{X}$, there exists a state $\tilde{\mathbf{x}}\in\Discretization$ such that $||\mathbf{x}-\tilde{\mathbf{x}}||_1<\tau$. Due to $\mathcal{X}$ being compact and therefore bounded, for any $\tau>0$ it is possible to compute its finite discretization with mesh $\tau$ by simply considering vertices of a grid with sufficiently small cells. Note that $f$, $\pi_{\theta}$ and $V_{\nu}$ are all continuous, hence due to $\mathcal{X}$ being compact $f$, $\pi_{\theta}$ and $V_{\nu}$ are also Lipschitz continuous. This will allow us to verify that the Expected decrease condition is satisfied by checking a slightly stricter condition only at the vertices of the discretization grid.
\begin{figure}
\centering
    \includegraphics[width=8.3cm]{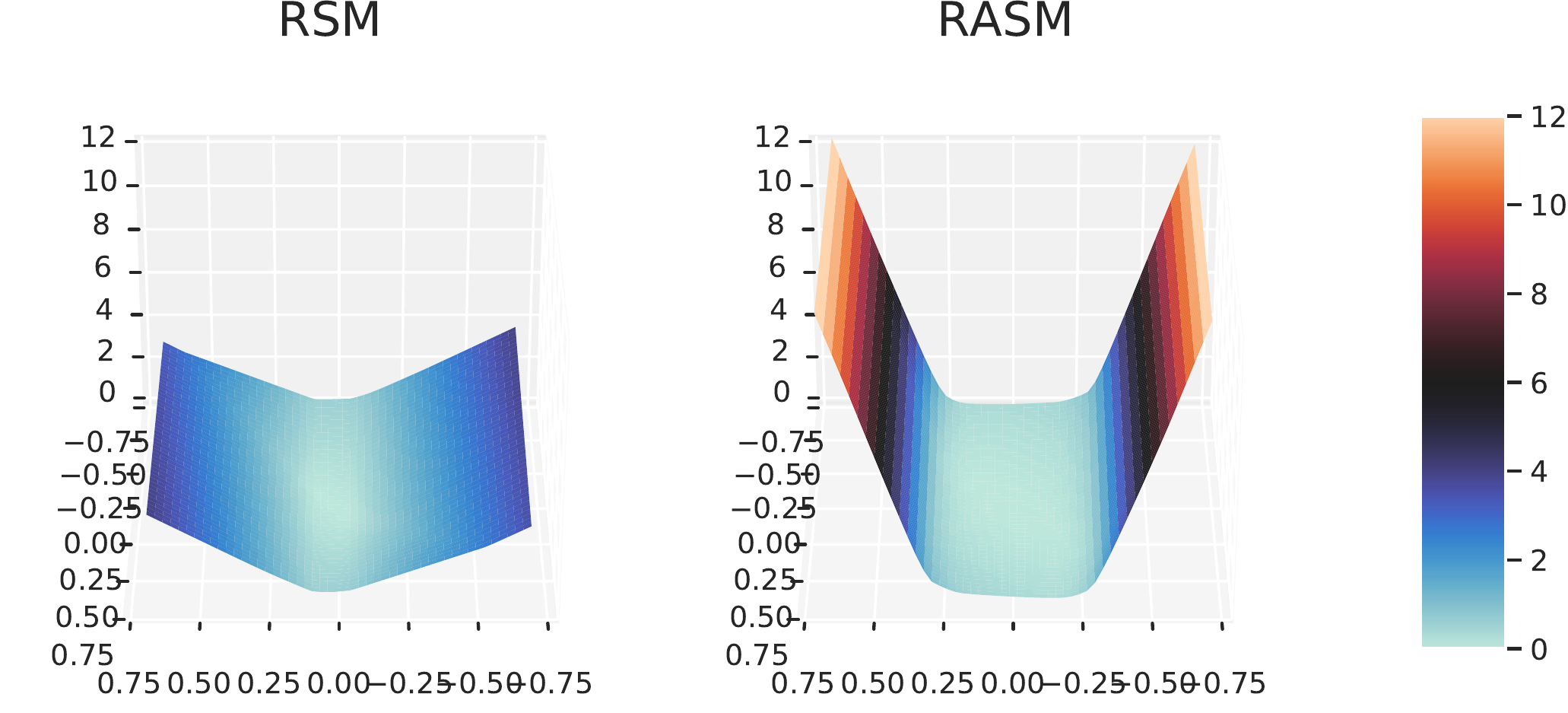}
  \caption{Visualization of a neural network RSM \cite{lechner2021stability} and our RASM on the inverted pendulum task. The RASM provides better probability bounds of reaching the unsafe states.}
    \label{fig:rasm}
\end{figure}
The initial discretization $\tilde{\mathcal{X}}$ is also used to initialize counterexample sets used by the learner. In particular, the learner initializes three sets $C_{\text{init}}= \tilde{\mathcal{X}}\cap\Init$, $C_{\text{unsafe}}=\tilde{\mathcal{X}}\cap\Unsafe$ and $C_{\text{decrease}}=\tilde{\mathcal{X}}\cap(\mathcal{X}\backslash\Target)$. These sets will later be extended by counterexamples computed by the verifier. Conversely, the discretization used by the verifier for checking the defining properties of RASMs will at each iteration of the loop be refined by a discretization with a smaller mesh, in order to relax the conditions that are checked by the verifier.

\smallskip\noindent{\bf Verifier}
We now describe the verifier module of our algorithm. Suppose that the learner has learned a policy $\pi_{\theta}$ and an RASM candidate $V_{\nu}$. Since $V_{\nu}$ is a neural network, we know that it is a continuous function. Furthermore, we design the learner to apply a softplus activation function to the output layer of $V_{\nu}$, which ensures that the Nonnegativity condition of RASMs is satisfied by default. Thus, the verifier only needs to check the Initial, Safety and Expected decrease conditions in Definition~\ref{def:rssm}.

Let $L_f$, $L_\pi$ and $L_V$ be the Lipschitz constants of $f$, $\pi_{\theta}$ and $V_{\nu}$, respectively. We assume that a Lipschitz constant for the dynamics function $f$ is provided, and use the method of~\citep{SzegedyZSBEGF13} to compute Lipschitz constants of neural networks $\pi_{\theta}$ and $V_{\nu}$. To verify the Expected decrease condition, the verifier collects the superset $\tilde{\mathcal{X}_e}$ of discretization points whose adjacent grid cells contain a non-target state and over which $V_{\nu}$ attains a value that is smaller than $\frac{1}{1-p}$. This set is computed by first collecting all cells that intersect $\mathcal{X}\backslash\Target$, then using interval arithmetic abstract interpretation (IA-AI)~\citep{CousotC77,Gowal18} which propagates interval bounds across neural network layers in order to bound from below the minimal value that $V_{\nu}$ attains over each collected cell, and finally collecting vertices of all cells at which this lower bound is less than $\frac{1}{1-p}$. The verifier then checks a stricter condition for each state $\tilde{\mathbf{x}}\in\tilde{\mathcal{X}_e}$:
\begin{equation}\label{eq:expdecstricter}
   \mathbb{E}_{\omega\sim d}\Big[ V_{\nu} \Big( f(\tilde{\mathbf{x}}, \pi_{\theta}(\tilde{\mathbf{x}}), \omega) \Big) \Big] < V_{\nu}(\tilde{\mathbf{x}}) - \tau \cdot K,
\end{equation}
where $K=L_V \cdot (L_f \cdot (L_\pi + 1) + 1)$. The expected value in eq.~\eqref{eq:expdecstricter} is also bounded from above via IA-AI, where one partitions the support of $d$ into intervals, propagates intervals and multiplies each interval bound by its probability weight in order to bound the expected value of a neural network function over a probability distribution. Due to space restrictions, we provide more details on expected value computation in the Appendix and note that this method requires that the probability distribution $d$ either has bounded support or is a product of independent univariate distributions.

In order to verify the Initial condition, the verifier collects the set $\text{Cells}_{\Init}$ of all cells of the discretization grid that intersect the initial set $\Init$. Then, for each $\text{cell}\in \text{Cells}_{\Init}$, it checks whether
\begin{equation}\label{eq:initialstricter}
    \sup_{\mathbf{x}\,\in\,\text{cell}}V_{\nu}(\mathbf{x}) > 1,
\end{equation}
where the supremum of $V_{\nu}$ over the cell is bounded from above by using IA-AI. Similarly, to verify the Unsafe condition, the verifier collects the set $\text{Cells}_{\Unsafe}$ of all cells of the discretization grid that intersect the unsafe set $\Unsafe$. Then, for each $\text{cell}\in \text{Cells}_{\Unsafe}$, it uses IA-AI to check whether
\begin{equation}\label{eq:unsafestricter}
    \inf_{\mathbf{x}\,\in\,\text{cell}}V_{\nu}(\mathbf{x}) < \frac{1}{1-p}.
\end{equation}

If the verifier shows that $V_{\nu}$ satisfies eq.~\eqref{eq:expdecstricter} for each $\tilde{\mathbf{x}}\in\tilde{\mathcal{X}_e}$, eq.~\eqref{eq:initialstricter} for each $\text{cell}\in \text{Cells}_{\Init}$ and eq.~\eqref{eq:unsafestricter} for each $\text{cell}\in \text{Cells}_{\Unsafe}$, it concludes that $V_{\nu}$ is an RASM. Otherwise, if a counterexample $\tilde{\mathbf{x}}$ to eq.~\eqref{eq:expdecstricter} is found and we have $\tilde{\mathbf{x}}\in\mathcal{X}\backslash\Target$ and $V_{\nu}(\mathbf{x})<\frac{1}{1-p}$, it is added to $C_{\text{decrease}}$. Similarly, if counterexample cells to eq.~\eqref{eq:initialstricter} and eq.~\eqref{eq:unsafestricter} are found, all their vertices that are contained in $\Init$ and $\Unsafe$ are added to $C_{\text{init}}$ and $C_{\text{unsafe}}$, respectively.

The following theorem shows that checking the above conditions is sufficient to formally verify whether an RASM candidate is indeed an RASM. The proof follows by exploiting the fact that $f$, $\pi_{\theta}$ and $V_{\nu}$ are all Lipschitz continuous and that $\mathcal{X}$ is compact, and we include it in the Appendix.

\begin{theorem}\label{theorem:verifier}
Suppose that the verifier verifies that $V_{\nu}$ satisfies eq.~\eqref{eq:expdecstricter} for each $\tilde{\mathbf{x}}\in\tilde{\mathcal{X}_e}$, eq.~\eqref{eq:initialstricter} for each $\text{cell}\in \text{Cells}_{\Init}$ and eq.~\eqref{eq:unsafestricter} for each $\text{cell}\in \text{Cells}_{\Unsafe}$. Then the function $V_{\nu}$ is an RASM for the system with respect to $\Target$, $\Unsafe$ and $p$.
\end{theorem}

\smallskip\noindent{\bf Learner}
A policy and an RASM candidate are learned by minimizing the loss function
\begin{equation*}
\begin{split}
    \loss(\theta, \nu) = &\loss_{\text{Init}}(\nu) + \loss_{\text{Unsafe}}(\nu) + \loss_{\text{Decrease}}(\theta,\nu)  \\
    &+ \lambda \cdot\big(\loss_{\textrm{Lipschitz}}(\theta) + \loss_{\textrm{Lipschitz}}(\nu)\big).
\end{split}
\end{equation*}
The first three loss terms are used to guide the learner towards learning a true RASM by forcing the learned candidate towards satisfying the Initial, Safety and Expected decrease conditions in Definition~\ref{def:rssm}. They are defined as follows:
\begin{equation*}
\begin{split}
    &\loss_{\text{Init}}(\nu) = \max_{\mathbf{x} \in C_{\text{init}}} \{V_\nu(\mathbf{x})-1, 0 \} \\
    &\loss_{\text{Unsafe}}(\nu) = \max_{\mathbf{x} \in C_{\text{unsafe}}}\{\frac{1}{1-p}-V_\nu(\mathbf{x}),0 \} \\
    &\loss_{\text{Decrease}}(\theta,\nu)  = \frac{1}{|C_{\text{decrease}}|} \cdot \\
    &\sum_{\mathbf{x}\in C_{\text{decrease}}}\Big( \max\Big\{ \sum_{\omega_1,\dots, \omega_N \sim \mathcal{N}}\frac{V_{\nu}\big(f(\mathbf{x},\pi_\theta(\mathbf{x}),\omega_i)\big)}{N} \\
    &\hspace{1cm} -  V_{\theta}(\mathbf{x})  + \tau \cdot K, 0\Big\} \Big)
\end{split}
\end{equation*}
Each loss term is designed to incur a loss at a state whenever that state violates the corresponding condition in Definition~\ref{def:rssm} that needs to be checked by the verifier. 
In the expression for $\loss_{\text{Decrease}}(\theta,\nu)$, we approximate the expected value of $V_{\nu}$ by taking the mean value of $V_{\nu}$ at $N$ sampled successor states, where $N\in\mathbb{N}$ is an algorithm parameter. This is necessary as it is not possible to compute a closed form expression for the expected value of a neural network~$V_{\nu}$.

The last loss term $\lambda \cdot (\loss_{\textrm{Lipschitz}}(\theta) + \loss_{\textrm{Lipschitz}}(\nu))$ is the regularization term used to guide the learner towards a policy and an RASM candidate with Lipschitz constants below a tolerable threshold $\rho$, with $\lambda>0$ being a regularization constant. By preferring networks with small Lipschitz constants, we allow the verifier to use a wider mesh, which significantly speeds up the verification process. The regularization term for $\pi_{\theta}$ (and analogously for $V_{\nu}$) is defined~via
\begin{equation*}
    \loss_{\text{Lipschitz}}(\theta) = \max\Big\{  \prod_{W,b \in \theta} \max_j \sum_{i} |W_{i,j}| - \rho, 0 \Big\},
\end{equation*}
where $W$ and $b$ weight matrices and bias vectors for each layer in $\pi_\theta$. Finally, in our implementation we also add an auxiliary loss term that does not enforce any of the defining conditions of RASMs, however it is used to guide the learner towards a candidate that attains the global minimum in a state that is contained within the target set $\Target$. We empirically observed that this term sometimes helps the updated policy from diverging from its objective to stabilize the system. Due to space restrictions, details are provided in the Appendix.

We remark that the loss function 
is always nonnegative but is not necessarily equal to $0$ even if $V_{\nu}$ satisfies all conditions checked by the verifier and if Lipschitz constants are below the specified thresholds. This is because the expected values in $\loss_{\text{Decrease}}(\theta,\nu)$ are approximated via sample means. However, in the following theorem we show that in this case $\loss(\theta,\nu) \rightarrow 0$ with probability $1$ as we add independent samples. 
The claim follows from the Strong Law of Large Numbers and the proof can be found in the Appendix.

\begin{theorem}\label{thm:loss}
Let $N$ be the number of samples used to approximate expected values in $\loss_{\text{Decrease}}(\theta,\nu)$. Suppose that $V_{\nu}$ satisfies eq.~\eqref{eq:expdecstricter} for each $\tilde{\mathbf{x}}\in\tilde{\mathcal{X}_e}$, eq.~\eqref{eq:initialstricter} for each $\text{cell}\in \text{Cells}_{\Init}$ and eq.~\eqref{eq:unsafestricter} for each $\text{cell}\in \text{Cells}_{\Unsafe}$. Suppose that Lipschitz constants of $\pi_{\theta}$ and $V_{\nu}$ are below the thresholds specified by $\loss_{\text{Lipschitz}}(\theta)$ and $\loss_{\text{Lipschitz}}(\nu)$ and that the samples in $\loss_{\text{Decrease}}(\theta,\nu)$ are independent. Then $\lim_{N\rightarrow \infty} \loss(\theta,\nu) = 0$ with probability $1$.
\end{theorem}

\begin{table}
    \centering
    \begin{tabular}{c|cc}\toprule
         & RSM & RASM \\
         Environment & (reach-avoid extension) & (ours) \\\midrule
         2D system & 83.4\% & 93.3\%\\
         Inverted pendulum & 47.9\% & 92.1\% \\
         Collision avoidance & Fail &  90.4\% \\\bottomrule
    \end{tabular}
    \caption{Reach-avoid probability obtained by our method and by the naive extension of RSMs. In each case, we report the largest reach-avoid probability successfully verified by the respective method.}
    \label{tab:rsm}
\end{table}   


\section{Experiments}

We experimentally validate our method on 3 non-linear RL environments.
Since no available baseline provides reach-avoid guarantees of stochastic systems over the infinite time horizon, as well as sampling and discretization approaches can only reason over finite time horizons, we aim our experiment as a validation of algorithm \ref{alg:algorithm} in practice.
We will make our JAX \citep{jax2018github} implementation publicly available.

Our first two environments are a linear 2D system with non-linear control bounds and the stochastic inverted pendulum control problem. 
The linear 2D system is of the form $\mathbf{x}_{t+1} = A\mathbf{x}_t + B g(\mathbf{u}_t) + \omega_t$, where $g: u \mapsto\min(\max(u,-1,1))$ limits the admissible action of the policy and $\omega_t$ is sampled from a triangular noise distribution.
The inverted pendulum environment is taken from the OpenAI Gym \citep{gym} and made more difficult by adding noise perturbations to its state. 
Our third environment concerns a collision avoidance task. The objective of this environment is to navigate an agent to the target region while avoiding crashing into one of two obstacles. All environments express bounds on the admissible actions. Further details of all environments can be found in the Appendix.

The policy and RASM networks consist of two hidden layers (128 units each, ReLU). The RASM network has a single output unit with a softplus activation.  We run our algorithm with a timeout of 3 hours.

The goal of our first experiment is to empirically evaluate the ability of our approach to learn probabilistic reach-avoid policies and to understand the importance of combining reachability with level set reasoning towards safety in stochastic systems. For all tasks, we pre-train the policy networks using 100 iterations of PPO. To evaluate our approach, we run our algorithm with several probability thresholds and report the highest threshold for which a policy together with an RASM is successfully learned. In order to understand the importance of simultaneous reasoning about reachability and level sets, we then compare our approach with a much simpler extension of the method of~\citep{lechner2021stability} which learns RSMs to certify probability $1$ reachability but does not consider any form of safety specifications. In particular, we run the method of~\citep{lechner2021stability} without the safety constraint and, in case a valid RSM is found, we normalize the function such that the Nonnegativity and the Initial conditions of RASMs are satisfied. We then bound from below the smallest value that the RSM attains over the unsafe region, and extract the corresponding reach-avoid probability bound according to the Safety condition of RASMs. Note that, even though this extension also exploits the ideas behind the level set reasoning in our RASMs, it {\em first} performs reachability analysis and only {\em afterwards} considers safety.
We remark that there is no existing method that provides reach-avoid guarantees of stochastic systems over the infinite time horizon, i.e. there is no existing baseline to compare against, thus we compare our level set reasoning with the extension of~\citep{lechner2021stability} which is the closest related work.

Table \ref{tab:rsm} shows results of our first experiment. In particular, in the third column we see that our method successfully learns policies that provide high probability reach-avoid guarantees for all benchmarks. On the other hand, comparison to the second column shows that {\em simultaneous} reasoning about reachability and safety that is allowed by our RASMs provides significantly better probabilistic reach-avoid guarantees than when such reasoning is decoupled. Figure~\ref{fig:rasm} visualizes the RSM computed by the baseline and our RASM.

\begin{table}
    \centering
    \begin{tabular}{c|cc}\toprule
         &  $V_\nu$ &$V_\nu$ and $\pi_\theta$ \\\midrule
         2D system & Fail (10 iters.) & 96.7\% (4 iters.) \\
        Collision avoidance & Fail (9 iters.) & 80.9\% (3 iters.)\\
         Inverted pendulum & Fail (7 iters.) & Fail (7 iters.) \\\bottomrule
    \end{tabular}
    \caption{Reach-avoid probabilities obtained by repairing unsafe policies. Verifying a policy by only learning the RASM $V_\nu$ times out, while jointly optimizing $V_\nu$ and $\pi_\theta$ yields a valid RASM. In each case, we report the largest reach-avoid probability successfully verified by the respective method.}
    \label{tab:frozen}
\end{table}   

In our second experiment, we study how well our algorithm can \emph{repair} (or \emph{fine-tune}) an unsafe policy. In particular, we pre-train the policy network using only 20 PPO iterations. We then run our algorithm with fixed policy parameters $\theta$, i.e.~we only learn an RASM in order to verify a probabilistic reach-avoid guarantee provided by a pre-trained policy. Next, we run our Algorithm~\ref{alg:algorithm} with both $\nu$ and $\theta$ as trainable parameters. Table \ref{tab:frozen} shows that, compared to a standalone verification method, our algorithm is able to repair unsafe policies in practice. However, the inability to repair the inverted pendulum policy illustrates that a decent starting policy is necessary for our algorithm, emphasizing the importance of policy initialization. Since the Policy Initialization step in Algorithm~\ref{alg:algorithm} initialises the policy by using PPO with a reward function that encodes the reach-avoid specification, our second experiment also demonstrates that a policy initialised by using RL on a tailored reward function is not sufficient to learn a reach-avoid policy with guarantees and that the learned policy requires “correction” in order to provide reach-avoid guarantees. The “correction” is achieved precisely by keeping the policy parameters trainable in the learner-verifier framework and fine-tuning them.

\section{Conclusion}

In this work, we present a method for learning controllers for discrete-time stochastic non-linear dynamical systems with formal reach-avoid guarantees. Our method learns a policy together with a reach-avoid supermartingale (RASM), a novel notion that we introduce in this work. It solves several important problems, including control with reach-avoid guarantees, verification of reach-avoid properties for a fixed policy, or fine-tuning of a given policy that does not satisfy a reach-avoid property. We demonstrated the effectiveness of our approach on three RL benchmarks. An interesting future direction would be to study certified control and verification of more general properties in stochastic systems. Since the aim of AI safety and formal verification is to ensure that systems do not behave in undesirable ways and that safety violating events are avoided, we are not aware of any potential negative societal impacts of our work.

\bibliographystyle{plain}

\bibliography{aaai23}

\section{Acknowledgments}
This work was supported in part by the ERC-2020-AdG 101020093, ERC CoG 863818 (FoRM-SMArt) and the European Union’s Horizon 2020 research and innovation programme under the Marie Skłodowska-Curie Grant Agreement No.~665385.
Research was sponsored by the United States Air Force Research Laboratory and the United States Air Force Artificial Intelligence Accelerator and was accomplished under Cooperative Agreement Number FA8750-19-2-1000. The views and conclusions contained in this document are those of the authors and should not be interpreted as representing the official policies, either expressed or implied, of the United States Air Force or the U.S. Government. The U.S. Government is authorized to reproduce and distribute reprints for Government purposes notwithstanding any copyright notation herein. The research was also funded in part by the AI2050 program at Schmidt Futures (Grant G-22-63172) and Capgemini SE.

\newpage
\begin{center}
	\Large\textbf{Appendix}
\end{center}

\section{Overview of Martingale Theory}\label{sec:martingales}

\paragraph{Probability theory} A {\em probability space} is a triple $(\Omega,\mathcal{F},\mathbb{P})$, where $\Omega$ is a non-empty {\em sample space}, $\mathcal{F}$ is a $\sigma$-algebra over $\Omega$ (i.e.~a collection of subsets of $\Omega$ that contains the empty set $\emptyset$ and is closed under complementation and countable union operations), and $\mathbb{P}$ is a {\em probability measure} over $\mathcal{F}$, i.e.~a function $\mathbb{P}:\mathcal{F}\rightarrow[0,1]$ that satisfies Kolmogorov axioms~\cite{Williams91}. We call the elements of $\mathcal{F}$ {\em events}. Given a probability space $(\Omega,\mathcal{F},\mathbb{P})$, a {\em random variable} is a function $X:\Omega\rightarrow\mathbb{R}$ that is $\mathcal{F}$-measurable, i.e.~for each $a\in\mathbb{R}$ we have that $\{\omega\in\Omega\mid X(\omega)\leq a\}\in\mathcal{F}$. We use $\mathbb{E}[X]$ to denote the {\em expected value} of $X$. A {\em (discrete-time) stochastic process} is a sequence $(X_i)_{i=0}^{\infty}$ of random variables in $(\Omega,\mathcal{F},\mathbb{P})$.
	
\paragraph{Conditional expectation} In order to formally define supermartingales, we need to introduce conditional expectation. Let $X$ be a random variable in a probability space $(\Omega,\mathcal{F},\mathbb{P})$. Given a sub-$\sigma$-algebra $\mathcal{F}'\subseteq\mathcal{F}$, a {\em conditional expectation} of $X$ given $\mathcal{F}'$ is an $\mathcal{F}'$-measurable random variable $Y$ such that, for each $A\in\mathcal{F}'$, we have 
\[ \mathbb{E}[X\cdot\mathbb{I}(A)]=\mathbb{E}[Y\cdot\mathbb{I}(A)].\]
The function $\mathbb{I}(A):\Omega\rightarrow \{0,1\}$ is an {\em indicator function} of $A$, defined via $\mathbb{I}(A)(\omega)=1$ if $\omega\in A$, and $\mathbb{I}(A)(\omega)=0$ if $\omega\not\in A$. Intuitively, conditional expectation of $X$ given $\mathcal{F}'$ is an $\mathcal{F}'$-measurable random variable that behaves like $X$ whenever its expected value is taken over an event in $\mathcal{F}'$. It is known that a conditional expectation of a random variable $X$ given $\mathcal{F}'$ exists if $X$ is real-valued and nonnegative~\cite{Williams91}. Moreover, for any two $\mathcal{F}'$-measurable random variables $Y$ and $Y'$ which are conditional expectations of $X$ given $\mathcal{F}'$, we have that $\mathbb{P}[Y= Y']=1$. Therefore, the conditional expectation is almost-surely unique and we may pick any such random variable as a canonical conditional expectation and denote it by $\mathbb{E}[X\mid \mathcal{F}']$.

\paragraph{Supermartingales} We are now ready to define supermartingales. Let $(\Omega,\mathcal{F},\mathbb{P})$ be a probability space and $(\mathcal{F}_i)_{i=0}^\infty$ be an increasing sequence of sub-$\sigma$-algebras in $\mathcal{F}$, i.e.~$\mathcal{F}_0\subseteq\mathcal{F}_1\subseteq\dots\subseteq\mathcal{F}$. A nonnegative {\em supermartingale} with respect to $(\mathcal{F}_i)_{i=0}^\infty$ is a stochastic process $(X_i)_{i=0}^{\infty}$ such that each $X_i$ is $\mathcal{F}_i$-measurable, and $X_i(\omega)\geq 0$ and $\mathbb{E}[X_{i+1}\mid\mathcal{F}_i](\omega) \leq X_i(\omega)$ hold for each $\omega\in\Omega$ and $i\geq 0$. Intuitively, the second condition says that the expected value of $X_{i+1}$ given the value of $X_i$ has to decrease, and this requirement is formally captured via conditional expectation.

We now present two results that will be key ingredients in the proof of Theorem~1. The first is Doob's Supermartingale Convergence Theorem (see~\cite{Williams91}, Section~11) which shows that every nonnegative supermartingale converges almost-surely to some finite value. The second theorem (see~\cite{Kushner14}, Theorem~7.1) provides a bound on the probability that the value of the supemartingale ever exceeds some threshold, and it will allow us to reason about both probabilistic reachability and safety. This is a less standard result from martingale theory, so we prove it below. In what follows, let $(\Omega,\mathcal{F},\mathbb{P})$ be a probability space and $(\mathcal{F}_i)_{i=0}^\infty$ be an increasing sequence of sub-$\sigma$-algebras in $\mathcal{F}$. 

\begin{theorem}[Supermartingale convergence theorem]\label{thm:convergence}
Let $(X_i)_{i=0}^{\infty}$ be a nonnegative supermartingale with respect to $(\mathcal{F}_i)_{i=0}^\infty$. Then, there exists a random variable $X_{\infty}$ in $(\Omega,\mathcal{F},\mathbb{P})$ to which the supermartingale converges to with probability $1$, i.e.~$\mathbb{P}[\lim_{i\rightarrow\infty}X_i=X_{\infty}]=1$.
\end{theorem}

\begin{theorem}\label{thm:bound}
Let $(X_i)_{i=0}^{\infty}$ be a nonnegative supermartingale with respect to $(\mathcal{F}_i)_{i=0}^\infty$. Then, for every $\lambda>0$, we have
\[ \mathbb{P}\Big[ \sup_{i\geq 0}X_i \geq \lambda \Big] \leq \frac{\mathbb{E}[X_0]}{\lambda}. \]
\end{theorem}

\begin{proof}
Fix $\lambda>0$. Define a stopping time $T:\Omega\rightarrow\mathbb{N}_0\cup\{\infty\}$ via $T = \inf_{i\in\mathbb{N}_0} \{X_i\geq \lambda\}$. Then, for each $n\in\mathbb{N}_0$, define a random variable
\[ X_{T\land n} = X_T\cdot\mathbb{I}(T \leq n) + X_n \cdot \mathbb{I}(T > n) \]
where $X_T$ is a random variable defined via $X_T(\omega)=X_{T(\omega)}(\omega)$ for each $\omega\in\Omega$, and $\mathbb{I}$ is again the indicator function that we defined in Section~\ref{sec:martingales}. It is a classical result from martingale theory that, for any $n\in\mathbb{N}_0$, we have $\mathbb{E}[X_{T\land n}]\leq \mathbb{E}[X_0]$ (see~\cite{Williams91}, Section 10.9). Hence, in order to prove the desired inequality, it suffices to prove $\lambda\cdot\mathbb{P}[ \sup_{i\geq 0}X_i \geq \lambda]\leq \sup_{n\in\mathbb{N}_0}\mathbb{E}[X_{T\land n}]$.

To prove the desired inequality we observe that, for each $n\in\mathbb{N}_0$, we have that
\begin{equation}\label{eq:stopping}
\begin{split}
    \mathbb{E}[X_{T\land n}] &= \mathbb{E}[X_T\cdot\mathbb{I}(T \leq n)] + \mathbb{E}[X_n \cdot \mathbb{I}(T > n)] \\
    &\geq \mathbb{E}[\lambda\cdot\mathbb{I}(T \leq n)] + \mathbb{E}[X_n \cdot \mathbb{I}(T > n)] \\
    &= \lambda\cdot\mathbb{P}[T \leq n] + \mathbb{E}[X_n \cdot \mathbb{I}(T > n)] \\
    &\geq \lambda\cdot\mathbb{P}[T \leq n] = \lambda\cdot\mathbb{P}[\sup_{0\leq i\leq n}X_i\geq \lambda].
\end{split}
\end{equation}
where in the first inequality we use the fact that $X_T\geq\lambda$, in the second inequality we use the fact that each $X_n$ is nonnegative and in the last equality we use the fact that $T = \inf_{i\in\mathbb{N}_0} \{X_i\geq \lambda\}$. Finally, $(\mathbb{P}[\sup_{0\leq i\leq n}X_i\geq \lambda])_{n=0}^\infty$ is a sequence of probabilities of events that are increasing with respect to set inclusion, so by the Monotone Convergence Theorem (see~\cite{Williams91}, Section 5.3) it follows that
\[ \lim_{n\rightarrow\infty}\mathbb{P}[\sup_{1\leq i\leq n}X_i\geq \lambda] = \mathbb{P}[\sup_{i\in\mathbb{N}_0}X_i\geq \lambda]. \]
Hence, by taking the supremum over $n\in\mathbb{N}_0$ of both sides of eq.~\eqref{eq:stopping}, we conclude that $\lambda\cdot\mathbb{P}[ \sup_{i\geq 0}X_i \geq \lambda]\leq \sup_{n\in\mathbb{N}_0}\mathbb{E}[X_{T\land n}]$, as desired. This concludes the proof of Theorem~\ref{thm:bound} since
\[ \mathbb{P}\Big[ \sup_{i\geq 0}X_i \geq \lambda \Big] \leq \sup_{n\in\mathbb{N}_0}\frac{\mathbb{E}[X_{T\land n}]}{\lambda} \leq \frac{\mathbb{E}[X_0]}{\lambda}. \]
\end{proof}

\section{Proof of Theorem~1}\label{sec:thmproof}

We now prove Theorem~1. Fix an initial state $\mathbf{x}_0\in\Init$ so that we need to show $\mathbb{P}_{\mathbf{x}_0}[ \ReachSafe(\Target,\Unsafe) ] \geq p$. Let $V$ be an RASM with respect to $\Target$, $\Unsafe$ and $p\in[0,1)$ whose existence is assumed in the theorem. First, we show that $V$ gives rise to a supermartingale in the probability space $(\Omega_{\mathbf{x}_0},\mathcal{F}_{\mathbf{x}_0},\mathbb{P}_{\mathbf{x}_0})$ of all trajectories of the system that start in $\mathbf{x}_0$. Then, we use Theorem~\ref{thm:convergence} and Theorem~\ref{thm:bound} to prove probabilistic reachability and safety.

For each time step $t\in\mathbb{N}_0$, define $\mathcal{F}_{\mathbf{x}_0,t}\subseteq\mathcal{F}_{\mathbf{x}_0}$ to be a sub-$\sigma$-algebra that, intuitively, contains events that are defined in terms of the first $t$ states of the system. Formally, for each $j\in\mathbb{N}_0$, let $C_j:\Omega_{\mathbf{x}_0}\rightarrow \mathcal{X}$ assign to each trajectory $\rho=(\mathbf{x}_t,\mathbf{u}_t,\omega_t)_{t\in\mathbb{N}_0}\in\Omega_{\mathbf{x}_0}$ the $j$-th state $\mathbf{x}_j$ along the trajectory. We define $\mathcal{F}_i$ to be the smallest $\sigma$-algebra over $\Omega_{\mathbf{x}_0}$ with respect to which $C_0, C_1, \dots, C_i$ are all measurable, where $\mathcal{X}\subseteq\mathbb{R}^m$ is equipped with the induced subset Borel-$\sigma$-algebra. The sequence $(\mathcal{F}_{\mathbf{x}_0,t})_{t=0}^\infty$ is increasing with respect to set inclusion.

Now, define a stochastic process $(X_t)_{t=0}^\infty$ in the probability space $(\Omega_{\mathbf{x}_0},\mathcal{F}_{\mathbf{x}_0},\mathbb{P}_{\mathbf{x}_0})$ via
\begin{equation*}
X_t(\rho) = \begin{cases}
    V(\mathbf{x}_t), &\text{if } \mathbf{x}_i\not\in\Target \text{ and } V(\mathbf{x}_i) < \frac{1}{1-p}\\
    &\text{for each } 0\leq i \leq t\\
    0, &\text{if } \mathbf{x}_i\in\Target \text{ for some }0\leq i\leq t \text{ and}\\
    &V(\mathbf{x}_j) < \frac{1}{1-p} \text{ for each }0\leq j\leq i\\
    \frac{1}{1-p}, &\text{otherwise}
\end{cases}
\end{equation*}
for each $t\in\mathbb{N}_0$ and a trajectory $\rho=(\mathbf{x}_t,\mathbf{u}_t,\omega_t)_{t\in\mathbb{N}_0}$. In other words, the value of $X_t$ is equal to the value of $V$ at $\mathbf{x}_t$, unless either the target set $\Target$ has been reached first in which case we set all future values of $\mathcal{X}_t$ to $0$, or a state in which $V$ exceeds $\frac{1}{1-p}$ has been reached first in which case we set all future values of $\mathcal{X}_t$ to $\frac{1}{1-p}$. We claim that $(X_t)_{t=0}^\infty$ is a nonnegative supermartingale with respect to $(\mathcal{F}_{\mathbf{x}_0,t})_{t=0}^\infty$. Indeed, each $X_t$ is $\mathcal{F}_{\mathbf{x}_0,t}$-measurable as it is defined in terms of the first $t$ states along a trajectory. It is also nonnegative as $V$ is nonnegative by the Nonnegativity condition of RASMs. Finally, to see that $\mathbb{E}_{\mathbf{x}_0}[X_{t+1}\mid\mathcal{F}_{\mathbf{x}_0,t}](\rho) \leq X_t(\rho)$ holds for each $t\in\mathbb{N}_0$ and $\rho=(\mathbf{x}_t,\mathbf{u}_t,\omega_t)_{t\in\mathbb{N}_0}$, we consider $3$ cases:
\begin{enumerate}
    \item If $\mathbf{x}_0,\mathbf{x}_1,\dots,\mathbf{x}_t\not\in\Target$ and $V(\mathbf{x}_i) < \frac{1}{1-p}$ for each $0\leq i\leq t$, then
    \begin{equation*}
    \begin{split}
        &\mathbb{E}_{\mathbf{x}_0}[X_{t+1}\mid\mathcal{F}_{\mathbf{x}_0,t}](\rho) \\
        &= \mathbb{E}_{\mathbf{x}_0}\Big[X_{t+1}\cdot \Big(\mathbb{I}(\mathbf{x}_{t+1}\not\in\Target \land V(\mathbf{x}_{t+1})<\frac{1}{1-p}) \\
        &\,\,\,+ \mathbb{I}(\mathbf{x}_{t+1}\in\Target) + \mathbb{I}(V(\mathbf{x}_{t+1})\geq \frac{1}{1-p})\Big)\mid\mathcal{F}_{\mathbf{x}_0,t}\Big](\rho) \\
        &= \mathbb{E}_{\mathbf{x}_0}[X_{t+1}\cdot \mathbb{I}(\mathbf{x}_{t+1}\not\in\Target)\mid\mathcal{F}_{\mathbf{x}_0,t}](\rho) \\
        &\,\,\,+ 0 + \frac{1}{1-p}\cdot\mathbb{E}[\mathbb{I}(V(\mathbf{x}_{t+1})\geq \frac{1}{1-p})\mid\mathcal{F}_{\mathbf{x}_0,t}](\rho)
    \end{split}
    \end{equation*}
    \begin{equation*}
    \begin{split}
        &\leq \mathbb{E}_{\omega\sim d}[V(f(\mathbf{x}_t,\mathbf{u}_t,\omega_t))\cdot\mathbb{I}(\mathbf{x}_{t+1}\not\in\Target \land V(\mathbf{x}_{t+1})<\frac{1}{1-p})] \\
        &\,\,\,+ \mathbb{E}_{\omega\sim d}[V(f(\mathbf{x}_t,\mathbf{u}_t,\omega_t))\cdot\mathbb{I}(\mathbf{x}_{t+1}\in\Target)] \\
        &\,\,\,+ \mathbb{E}_{\omega\sim d}[V(f(\mathbf{x}_t,\mathbf{u}_t,\omega_t))\cdot\mathbb{I}(V(\mathbf{x}_{t+1})\geq\frac{1}{1-p})] \\
        &= \mathbb{E}_{\omega\sim d}[V(f(\mathbf{x}_t,\mathbf{u}_t,\omega_t))] \\
        &\leq V(\mathbf{x}_t) - \eps.
    \end{split}
    \end{equation*}
    Here, the first equality follows by the law of total probability, the second equality follows by our definition of each $X_t$, the third inequality follows by observing that $V(\mathbf{x}_{t+1})\geq X_{t+1}(\rho)$ in this case, the fourth equality is just the sum of expectations over disjoint sets, and finally the fifth inequality follows by the Expected decrease condition in Definition~1 since $\mathbf{x}_t\not\in\Target$ and $V(\mathbf{x}_t)<\frac{1}{1-p}$, by the assumption of this case.
    
    \item If $\mathbf{x}_i\in\Target$ for some $0\leq i\leq t$ and $V(\mathbf{x}_j)<\frac{1}{1-p}$ for all $0\leq j\leq i$, then we have $\mathbb{E}_{\mathbf{x}_0}[X_{t+1}\mid\mathcal{F}_{\mathbf{x}_0,t}](\rho)=X_{t+1}(\rho)=0$.
    
    \item Otherwise, we must have $V(\mathbf{x}_i)\geq\frac{1}{1-p}$ and $\mathbf{x}_0,\dots,\mathbf{x}_i\not\in\Target$ for some $-\leq i\leq t$, thus $\mathbb{E}_{\mathbf{x}_0}[X_{t+1}\mid\mathcal{F}_{\mathbf{x}_0,t}](\rho)=X_{t+1}(\rho)=\frac{1}{1-p}$.
\end{enumerate}
Hence, we have proved that $(X_t)_{t=0}^\infty$ is a nonnegative supermartingale.

Now, by Theorem~\ref{thm:convergence} it follows that the value of the nonnegative supermartingale $(X_t)_{t=0}^\infty$ with probability $1$ converges. In what follows, we show that $(X_t)_{t=0}^\infty$ with probability $1$ converges to and reaches either $0$ or a value that is greater than or equal to $\frac{1}{1-p}$. To do this, we use the fact that the Expected decrease condition of RASMs enforces the value of $V$ to decrease in expected value by at least $\eps>0$ after every one-step evolution of the system in any non-target state at which $V(\mathbf{x}) < \frac{1}{1-p}$. Define the stopping time $T:\Omega_{\mathbf{x}_0}\rightarrow\mathbb{N}_0\cup\{\infty\}$ via
\[ T(\rho)=\inf_{t\in\mathbb{N}_0}\Big\{X_t(\rho)=0\lor X_t(\rho)\geq\frac{1}{1-p}\Big\}. \]
Our goal is then to prove that $\mathbb{P}_{\mathbf{x}_0}[T < \infty] = 1$. Using the argument in the proof that $(X_t)_{t=0}^{\infty}$ is a nonnegative supermartingale (in particular, the proof of supermartingale property in Case~$1$), we can in fact deduce a stronger inequality
\[ \mathbb{E}_{\mathbf{x}_0}[X_{t+1}\mid\mathcal{F}_{\mathbf{x}_0,t}](\rho)\leq X_t(\rho) - \eps\cdot \mathbb{I}(T(\rho) > t) \]
for each $\rho\in\Omega_{\mathbf{x}_0}$. But now, we may use Proposition~\ref{prop:convergence} stated below to deduce that $\mathbb{E}_{\mathbf{x}_0}[T]\leq \mathbf{E}_{\mathbf{x}_0}[X_0]=V(\mathbf{x}_0)<\infty$, which in turn implies that $\mathbb{P}_{\mathbf{x}_0}[T < \infty] = 1$, as desired. This concludes the proof.

The following proposition states a results on probability $1$ convergence of {\em ranking supermartingales (RSMs)}. RSMs are a notion similar to our RASMs that were first introduced in~\cite{ChakarovS13} in order to study termination in probabilistic programs, and were used in~\cite{lechner2021stability} to formally verify almost-sure stability and reachability in stochastic control systems. We note that RASMs generalize RSMs in the sense that RSMs coincide with RASMs in the special case when the unsafe set is empty and we only consider a probability~$1$ reachability specification, i.e.~$\Unsafe=\emptyset$.

\begin{proposition}[\cite{ChakarovS13}]\label{prop:convergence}
Let $(\Omega,\mathcal{F},\mathbb{P})$ be a probability space, let $(\mathcal{F}_i)_{i=0}^\infty$ be an increasing sequence of sub-$\sigma$-algebras in $\mathcal{F}$ and let $T$ be a stopping time with respect to $(\mathcal{F}_i)_{i=0}^\infty$. Suppose that $(X_i)_{i=0}^{\infty}$ is a stochastic process such that each $X_i$ is nonnegative and we have that 
\[ \mathbb{E}[X_{i+1}\mid\mathcal{F}_i](\omega)\leq X_i(\omega) - \eps\cdot \mathbb{I}(T(\omega) > i) \]
holds for each $i\in\mathbb{N}_0$ and $\omega\in\Omega$. Then $\mathbb{P}[T<\infty]=1$.
\end{proposition}


Finally, by using Theorem~\ref{thm:bound} for the nonnegative supermartingale $(X_t)_{t=0}^\infty$ and $\lambda=\frac{1}{1-p}>0$, it follows that $\mathbb{P}_{\mathbf{x}_0}[ \sup_{i\geq 0}X_i \geq \frac{1}{1-p} ] \leq (1-p)\cdot \mathbb{E}_{\mathbf{x}_0}[X_0] \leq 1-p$. The second inequality follows since $X_0(\rho) = V(\mathbf{x}_0)\leq 1$ for every $\rho\in\Omega_{\mathbf{x}_0}$ by the Initial condition of RASMs. Hence, as $(X_t)_{t=0}^\infty$ with probability $1$ either reaches $0$ or a value that is greater than or equal to $\frac{1}{1-p}$, we conclude that $(X_t)_{t=0}^\infty$ reaches $0$ without reaching a value that is greater than or equal to $\frac{1}{1-p}$ with probability at least $p$. By the definition of each $X_t$ and by the Safety condition of RASMs, this implies that with probability at least $p$ the system will reach the target set $\Target$without reaching the unsafe set $\Unsafe$, i.e.~that $\mathbb{P}_{\mathbf{x}_0}[ \ReachSafe(\Target,\Unsafe) ] \geq p$.

\section{Computation of Expected Values of Neural Networks}\label{sec:expvalcomputation}

We now describe the method for bounding the expected value of a neural network function over a given probability distribution. Let $\mathbf{x}\in\mathcal{X}$ be a fixed state, and suppose that we want to bound the expected value $\mathbb{E}_{\omega\sim d}[ V ( f(\mathbf{x}, \pi(\mathbf{x}), \omega) )]$. We partition the disturbance space $\mathcal{N}\subseteq\mathbb{R}^p$ into finitely many cells $\text{cell}(\mathcal{N}) = \{\mathcal{N}_1,\dots,\mathcal{N}_{k}\}$.
Let $\mathrm{maxvol}=\max_{\mathcal{N}_i\in \text{cell}(\mathcal{N})}\mathsf{vol}(\mathcal{N}_i)$ denote the maximal volume of any cell in the partition (with respect to the Lebesgue measure). The expected value is bounded via
\begin{equation*}
    \mathbb{E}_{\omega\sim d}\Big[ V \Big( f(\mathbf{x}, \pi(\mathbf{x}), \omega) \Big) \Big] \leq \sum_{\mathcal{N}_i\in \text{cell}(\mathcal{N})} \mathrm{maxvol} \cdot \sup_{\omega\in \mathcal{N}_i} F(\omega)
\end{equation*}
where $F(\omega) = V( f(\mathbf{x}, \pi(\mathbf{x}), \omega))$. Each supremum is then bounded from above via interval arithmetic by using the method of~\cite{Gowal18}. Note that $\mathrm{maxvol}$ is not finite if $\mathcal{N}$ is unbounded. In order to allow expected value computation for an unbounded $\mathcal{N}$ under the assumption that $d$ is a product of univariate distributions, the method first applies the probability integral transform~\cite{Murphy12} to each univariate probability distribution in $d$ in order to reduce the problem to the case of a probability distribution of bounded support.

\section{Proof of Theorem~2}\label{sec:proofthmverifier}

Suppose that the verifier verifies that $V$ satisfies eq.~(1) for each $\tilde{\mathbf{x}}\in\tilde{\mathcal{X}_e}$, eq.~(2) for each $\text{cell}\in \text{Cells}_{\Init}$ and eq.~(3) for each $\text{cell}\in \text{Cells}_{\Unsafe}$. The fact that the Initial and the Unsafe conditions in Definition~1 of RASMs are satisfied by $V$ then follows from the correctness of interval arithmetic abstract interpretation (IA-AI) of~\cite{Gowal18}. Thus, we only need to show that $V$ satisfies the Expected decrease condition in Definition~1.

To show that $V$ satisfies the Expected decrease condition, we need to show that there exists $\eps>0$ such that $V(\mathbf{x}) \geq \mathbb{E}_{\omega\sim d}[V(f(\mathbf{x},\pi(\mathbf{x}),\omega))] + \eps$ holds for all $\mathbf{x}\in\mathcal{X}\backslash\Target$ at which $V(\mathbf{x}) \leq \frac{1}{1-p}$.  We prove that $\eps>0$ defined via
\[ \eps = \min_{\tilde{\mathbf{x}}\in \tilde{\mathcal{X}_\eps}} \Big( V(\tilde{\mathbf{x}}) - \tau \cdot K - \mathbb{E}_{\omega\sim d}\Big[ V \Big( f(\tilde{\mathbf{x}}, \pi(\tilde{\mathbf{x}}), \omega) \Big) \Big] \Big) \]
satisfies this property. Note that $\eps>0$, as each $\tilde{\mathbf{x}}\in\tilde{\mathcal{X}_e}$ satisfies eq.~(1).

To show this, fix $\mathbf{x}\in\mathcal{X}\backslash\Target$ with $V(\mathbf{x})\leq \frac{1}{1-p}$ and let $\tilde{\mathbf{x}}\in\tilde{\mathcal{X}_e}$ be such that $||\mathbf{x}-\tilde{\mathbf{x}}||_1 \leq \tau$. By construction, the set $\tilde{\mathcal{X}_e}$ contains vertices of each discretization cell that intersects $\mathcal{X}\backslash\Target$ and that contains at least one state at which $V$ is less than or equal to $\frac{1}{1-p}$, hence such $\tilde{\mathbf{x}}$ exists. We then have
\begin{equation}\label{eq:long}
\begin{split}
    &\mathbb{E}_{\omega\sim d}\Big[ V \Big( f(\mathbf{x}, \pi(\mathbf{x}), \omega) \Big) \Big] \\
    &\leq \mathbb{E}_{\omega\sim d}\Big[ V \Big( f(\tilde{\mathbf{x}}, \pi(\tilde{\mathbf{x}}), \omega) \Big) \Big] \\
    &\hspace{1cm}+ ||f(\tilde{\mathbf{x}}, \pi(\tilde{\mathbf{x}}), \omega) - f(\mathbf{x}, \pi(\mathbf{x}), \omega)||_1 \cdot L_V \\
    &\leq \mathbb{E}_{\omega\sim d}\Big[ V \Big( f(\tilde{\mathbf{x}}, \pi(\tilde{\mathbf{x}}), \omega) \Big) \Big] \\
    &\hspace{1cm}+ ||(\tilde{\mathbf{x}}, \pi(\tilde{\mathbf{x}}), \omega) - (\mathbf{x}, \pi(\mathbf{x}), \omega)||_1 \cdot L_V\cdot L_f \\
    &\leq \mathbb{E}_{\omega\sim d}\Big[ V \Big( f(\tilde{\mathbf{x}}, \pi(\tilde{\mathbf{x}}), \omega) \Big) \Big] \\
    &\hspace{1cm}+ ||\tilde{\mathbf{x}} - \mathbf{x}||_1 \cdot L_V\cdot L_f \cdot (1 + L_{\pi}) \\
    &\leq \mathbb{E}_{\omega\sim d}\Big[ V \Big( f(\tilde{\mathbf{x}}, \pi(\tilde{\mathbf{x}}), \omega) \Big) \Big] \\
    &\hspace{1cm}+ \tau \cdot L_V\cdot L_f \cdot (1 + L_{\pi}).
\end{split}
\end{equation}
On the other hand, we also have
\begin{equation}\label{eq:short}
    V(\mathbf{x}) \geq V(\tilde{\mathbf{x}}) - ||\tilde{\mathbf{x}} - \mathbf{x}||_1 \cdot L_V \geq V(\tilde{\mathbf{x}}) - \tau\cdot L_V.
\end{equation}
Combining eq.(\ref{eq:long}) and (\ref{eq:short}), we conclude that
\begin{equation}
\begin{split}
    &V(\mathbf{x}) - \mathbb{E}_{\omega\sim d}\Big[ V \Big( f(\mathbf{x}, \pi(\mathbf{x}), \omega) \Big) \Big] \\
    &\geq V(\tilde{\mathbf{x}}) - \tau\cdot L_V - \mathbb{E}_{\omega\sim d}\Big[ V \Big( f(\tilde{\mathbf{x}}, \pi(\tilde{\mathbf{x}}), \omega) \Big) \Big] \\
    &\hspace{1cm}- \tau \cdot L_V\cdot L_f \cdot (1 + L_{\pi}) \\
    &= V(\tilde{\mathbf{x}}) - \tau\cdot K - \mathbb{E}_{\omega\sim d}\Big[ V \Big( f(\tilde{\mathbf{x}}, \pi(\tilde{\mathbf{x}}), \omega) \Big) \Big] \\
    &\geq \eps
\end{split}
\end{equation}
where the equality in the second last row follows by the definition of $K$, and the inequality in the last row follows by our choice of $\eps$. Hence, $V$ satisfies the Expected decrease condition and is indeed an RASM as in Definition~1.

\section{Auxiliary Loss Term}\label{sec:appauxiliary}

The loss term $\loss_{\text{Aux}}(\nu)$ is an auxiliary loss term that does not enforce any of the defining conditions of RASMs, however it is used to guide the learner towards a candidate that attains the global minimum in a state that is contained within the target set $\Target$. We empirically observed that this term prevents the updated policy from diverging from its objective to stabilize the system. It is defined via
\begin{equation*}
\begin{split}
    \loss_{\text{Aux}}(\nu) & = \max\{V_\nu (\hat{\mathbf{x}}_{\text{Target}})- \varepsilon,0 \} \\
    &+ \max\{\min_{\mathbf{x} \in \tilde{X}\cap \Target} V_\nu(\mathbf{x}) - \min_{\mathbf{x} \in C_{\text{init}}} V_\nu(\mathbf{x}),0 \} \\
    &+ \max\{\min_{\mathbf{x} \in \tilde{X}\cap \Target} V_\nu(\mathbf{x}) - \min_{\mathbf{x} \in C_{\text{unsafe}}} V_\nu(\mathbf{x}),0 \}
\end{split}
\end{equation*}
with $\tilde{\mathbf{x}}_{\text{Target}}$ being some state contained in the target set $\Target$ and $\varepsilon\geq 0$ an algorithm parameter.

\section{Proof of Theorem~3}\label{sec:proofthmlearner}

Suppose that $V_{\nu}$ satisfies eq.~(1) for each $\tilde{\mathbf{x}}\in\tilde{\mathcal{X}_e}$, eq.~(2) for each $\text{cell}\in \text{Cells}_{\Init}$ and eq.~(3) for each $\text{cell}\in \text{Cells}_{\Unsafe}$. Suppose that Lipschitz constants of $\pi_{\theta}$ and $V_{\nu}$ are below the thresholds specified by $\loss_{\text{Lipschitz}}(\theta)$ and $\loss_{\text{Lipschitz}}(\nu)$ and that the samples in $\loss_{\text{Decrease}}(\theta,\nu)$ are independent. We need to show that $\lim_{N\rightarrow \infty} \loss(\theta,\nu) = 0$ with probability $1$.

Since Lipschitz constants of $\pi_{\theta}$ and $V_{\nu}$ are below the thresholds specified by $\loss_{\text{Lipschitz}}(\theta)$ and $\loss_{\text{Lipschitz}}(\nu)$, we have that $\lambda\cdot (\loss_{\text{Lipschitz}}(\theta) + \loss_{\text{Lipschitz}}(\nu)) = 0$. Moreover, our initialization of $C_{\text{init}}$ and our design of the verifier module ensure that $C_{\text{init}}$ contains only states in $\Init$, hence $\loss_{\text{Init}}(\nu)=0$ as $V_{\nu}$ satisfies the Initial condition of RASMs. Note, $V_{\nu}$ satisfies all conditions checked by the verifier, hence by Theorem~2 we know that it is an RASM.  Similarly, $C_{\text{unsafe}}$ contains only states in $\Unsafe$, hence $\loss_{\text{Unsafe}}(\nu)=0$ as $V_{\nu}$ satisfies the Safety condition of RASMs. Thus, under theorem assumptions we have that
\begin{equation*}
\begin{split}
    &\loss(\theta,\nu) = \loss_{\text{Decrease}}(\theta,\nu) \\
    &= \frac{1}{|C_{\text{decrease}}|}\sum_{\mathbf{x}\in C_{\text{decrease}}} \Big( \max \\
    &\Big\{ \sum_{\omega_1,\dots, \omega_N \sim \mathcal{N}}\frac{V_{\nu}\big(f(\mathbf{x},\pi_\theta(\mathbf{x}),\omega_i)\big)}{N} -  V_{\theta}(\mathbf{x})  + \tau \cdot K, 0\Big\} \Big)
\end{split}
\end{equation*}

Hence, in order to prove that $\lim_{N\rightarrow \infty} \loss(\theta,\nu) = 0$ with probability $1$, it suffices to prove that for each $\mathbf{x}\in C_{\text{decrease}}$ with probability $1$ we have
\begin{equation*}
\begin{split}
    &\lim_{N\rightarrow\infty}\max\Big\{ \sum_{\omega_1,\dots, \omega_N \sim \mathcal{N}}\frac{V_{\nu}\big(f(\mathbf{x},\pi_\theta(\mathbf{x}),\omega_i)\big)}{N} \\
    &-  V_{\theta}(\mathbf{x}) + \tau \cdot K, 0\Big\} = 0.
\end{split}
\end{equation*}

The above sum is the mean of $N$ independently sampled successor states of $\mathbf{x}$, which are sampled according to the probability distribution defined by the system dynamics and the probability distribution $d$ over disturbance vectors. Since the state space of the system is assumed to be compact and $V_{\theta}$ is continuous as it is a neural network, the random value defined by the value of $V$ at a sampled successor state is bounded and therefore admits a well-defined and finite first moment. The Strong Law of Large Numbers~\cite{Williams91} then implies that the above sum converges to the expected value of this distribution as $N\rightarrow\infty$. Thus, with probability $1$, we have that\begin{equation*}
\begin{split}
    &\lim_{N\rightarrow\infty}\max\Big\{ \sum_{\omega_1,\dots, \omega_N \sim \mathcal{N}}\frac{V_{\nu}\big(f(\mathbf{x},\pi_\theta(\mathbf{x}),\omega_i)\big)}{N} -  V_{\theta}(\mathbf{x}) + \tau \cdot K, 0\Big\} \\
    &=\max\Big\{ \lim_{M\rightarrow\infty}\sum_{\omega_1,\dots, \omega_N \sim \mathcal{N}}\frac{V_{\nu}\big(f(\mathbf{x},\pi_\theta(\mathbf{x}),\omega_i)\big)}{N} -  V_{\theta}(\mathbf{x}) + \tau \cdot K, 0\Big\} \\
    &=\max\Big\{ \lim_{M\rightarrow\infty}\mathbb{E}_{\omega\sim d}\Big[f(\mathbf{x},\pi_\theta(\mathbf{x}),\omega)\big] -  V_{\theta}(\mathbf{x}) + \tau \cdot K, 0\Big\} \\
    &= 0.
\end{split}
\end{equation*}
The first equality holds since a limit may be interchanged with the maximum function over a finite number of arguments, the second equality holds with probability $1$ by the Strong Law of Large Numbers, and the third equality holds since $V_{\nu}$ satisfies eq.~(1) for each $\tilde{\mathbf{x}}\in\tilde{\mathcal{X}_e}$ and we have $C_{\text{decrease}}\subseteq \tilde{\mathcal{X}_e}$. This concludes our proof that $\lim_{N\rightarrow \infty} \loss(\theta,\nu) = 0$ with probability~$1$.

\section{Experiment details}\label{supp:experiment}
The dynamics of the 2D system is given by
\begin{equation}
     \mathbf{x}_{t+1} = \begin{pmatrix} 1 & 0.045 \\ 0 & 0.9 \end{pmatrix} \mathbf{x}_t + \begin{pmatrix} 0.45 \\0.5 \end{pmatrix} g(\mathbf{u}_t) + \begin{pmatrix}0.01 & 0 \\0 & 0.005 \end{pmatrix}  \omega, 
\end{equation}
with $\omega$ being the disturbance vector and $\omega[1],\omega[2] \sim \text{Triangular}$. Square brackets indicate the coordinate index, e.g.~$\omega[1]$ is the first coordinate of $\omega\in\mathbb{R}^2$. The probability density function of $\text{Triangular}$ is defined by 
\begin{equation}
    \text{Triangular}(x) := \begin{cases} 0 & \text{if } x< -1\\ 1 - |x| & \text{if } -1 \leq x \leq 1\\ 0 & \text{otherwise}\end{cases}.
\end{equation}
The function $g$ is defined as $g(u) = \max(\min(u,1),-1)$.
The state space of the 2D system is define as $\mathcal{X} = [-1.5,1.5]^2$. We define the target set as $\Target = [-0.2,0.2]^2$, the initial set $\Init = [-0.25,-0.2]\times [-0.1,0.1] \cup  [0.25,0.2]\times [-0.1,0.1] $ and the unsafe set $\Unsafe = [-1.5,-1.4]\times[-1.5,0] \cup  [1.4,1.5]\times[0,1.5]$.

For the inverted pendulum task, the dynamics is given by 
\begin{align*}
    \mathbf{x}_{t+1}[2] &:= (1-b)  \mathbf{x}_{t}[2] \\
    &+ \Delta \cdot \big( \frac{-1.5 \cdot G \cdot \text{sin}(\mathbf{x}_{t}[1]+\pi)}{2l} + \frac{3}{m l^2} 2g(\mathbf{u}_t)\big)\\
    &+ 0.02 \omega[1]\\
    \mathbf{x}_{t+1}[1] &:=  \mathbf{x}_{t}[1] + \Delta \cdot \mathbf{x}_{t+1}[2] + 0.01 \omega[2],
\end{align*}
with $\delta, G, m, l, b$ being defined in Table \ref{tab:invpend}.
The state space of the inverted pendulum environment is define as $\mathcal{X} = [-0.7,0.7]^2$. We define the target set as $\Target = [-0.2,0.2]^2$, the initial set $\Init = [-0.3,0.3]^2$ and the unsafe set $\Unsafe = [-0.7,-0.6]\times[-0.7,0] \cup  [0.6,0.7]\times[0,0.7]$.

\begin{table}[]
    \centering
    \begin{tabular}{c|c}\toprule
        Parameter & Value  \\\midrule
        $\Delta$ &  0.05\\
        $G$ & 10\\
        $m$ & 0.15\\
        $l$ & 0.5\\
        $b$ & 0.1\\\bottomrule
    \end{tabular}
    \caption{Parameters of the inverted pendulum task.}
    \label{tab:invpend}
\end{table}

The dynamics of the collision avoidance task is given by 
\begin{align*}
    d_1 & = \max\{0.33 ||\mathbf{x}_t - \begin{pmatrix}0\\ 1\end{pmatrix}|| ,0\}\\
    d_2 & = \max\{0.33 ||\mathbf{x}_t - \begin{pmatrix}0 \\-1\end{pmatrix}|| ,0\}\\
    \mathbf{x}_{t+1} & = \mathbf{x}_{t} + 0.2 \Big(d_2 \big(d_1 \mathbf{u}_t + (1-d_1)\begin{pmatrix}0 \\1\end{pmatrix}\big)  \\ & + (1-d_2) \begin{pmatrix}0\\ -1\end{pmatrix}\Big) + 0.05 \omega_t.
\end{align*}
The state space of the collision avoidance environment is define as $\mathcal{X} = [-1,1]^2$. We define the target set as $\Target = [-0.2,0.2]^2$, the initial set $\Init = [-1,-0.9]\times[-0.6,0.6] \cup  [0.9,1.0]\times[-0.6,0.6]$ and the unsafe set $\Unsafe = [-0.3,0.3]\times[0.7,1] \cup  [-0.3,0.3]\times[-1,-0.7]$. Similar to the first two environments, $\omega_t$ is a triangular distributed random variable.

We optimize the training objective using stochastic gradient descent. 
As mentioned in the main text, the policy and RASM networks consist of two hidden layers with 128 units each and ReLU activation function. The RASMs network has single output unit with a softplus activation, while the output dimension of the policy network depends on the task.

The used hyperparameters of our algorithm used in the experiments are listed in Table \ref{tab:hparams}. The code is available for review in the supplementary materials.

\begin{table}[]
    \centering
    \begin{tabular}{c|c}\toprule
    Hyperparameter & Value \\\midrule
         Optimizer & Adam \cite{kingma2014adam} \\
         Batch size & 4096\\
         $V_\nu$ learning rate & 5e-4 \\
    $\pi_\theta$ learning rate & 5e-5 \\
    Lipschitz factor $\gamma$ & 0.001\\
 Lipschitz threshold $\rho_\theta$  & 4\\
  Lipschitz threshold $\rho_\nu$  & 15\\
    $\tau$ (2D system) & 6e-3 \\
  $\tau$ (inverted pendulum) & 2.3e-3 \\
  $\tau$ (collision avoidance) & 4e-3 \\ 
    $\varepsilon$ & 0.3\\
  $N$ & 16 \\
  Number of cells for expectation & 144\\  \bottomrule
    \end{tabular}
    \caption{List of hyperparameter values used in the experiments.}
    \label{tab:hparams}
\end{table}

\textbf{Normalizing the learned RASM for better bounds}

After our algorithm has terminated we can slightly improve the probability bounds certified by the verifier.
In particular, the normalization linearly transforms $V_\nu$ such that the supremum of the new RASM $V'_\nu$ at the initial set is 1 and the infimum of $V'_\nu$ on the entire domain is 0, i.e.,
\begin{equation}
    V'_\nu(\mathbf{x}) = \frac{V_\nu(\mathbf{x})-\inf_{\mathbf{x}' \in \mathcal{X}} V_\nu(\mathbf{x}')}{\sup_{\mathbf{x}' \in \Init} V_\nu(\mathbf{x}')-\inf_{\mathbf{x}' \in \mathcal{X}} V_\nu(\mathbf{x}')}.
\end{equation}
The improved probability bounds $p'$ can then be computed according to condition 3 by 
\begin{equation}
    p' = 1 - \frac{1}{\inf_{\mathbf{x} \in \Unsafe} V'_\nu(\mathbf{x})}.
\end{equation}

The supremum and infimum are computed using our abstract interpretation (IA-AI) on the cell grids.

\textbf{PPO Details}\label{app:ppo}
Here, we list the settings used for the PPO pre-training process of the policy networks \cite{schulman2017proximal}.
In every PPO iteration we collect 30 episodes of the environment as training data in the experience buffer. The policy $\pi_\mu$ is made stochastic using a Gaussian distributed random variable that is added to the policy's output, i.e., the policy predicts the mean of the Gaussian. 
The standard deviation of the Gaussian is annealed during the policy training process, starting from 0.5 at first PPO iteration to 0.05 at PPO iteration 50. We normalize the advantage values, i.e., the difference between the observed discounted returns and the predicted return by the value function, by subtracting the mean and dividing by the standard deviation of the advantage values of the experience buffer. The PPO clipping value $\varepsilon$ is set to 0.2 and the discount factor $\gamma$ to 0.99.
In every PPO iteration, the policy is trained for 10 epochs, except for the first iteration where the network is trained for 30 epochs. An epoch corresponds to a pass over the entire data in the experience buffer, i.e., the data from the the 30 episodes. 
The value network is trained for 5 epochs, expect in the first PPO iteration, where the training is performed for 10 epochs. We apply the Lipschitz regularization on the policy parameters already during the PPO pre-training of the policy.

\end{document}